\documentclass{article}
\usepackage[utf8]{inputenc}
\usepackage[english]{babel}

\usepackage[tbtags]{amsmath}
\usepackage{amsfonts,amssymb,mathrsfs,amscd,comment,amsthm}
\usepackage{natbib}

\usepackage{url}  
\usepackage{subcaption}
\usepackage{booktabs}       
\usepackage{amsfonts}       
\usepackage{nicefrac}       
\usepackage{microtype}      
\usepackage{comment}
\usepackage{amssymb, amsmath, latexsym}
\usepackage{url}

\usepackage{algorithm}
\usepackage{algorithmic}

\usepackage{tabularx}
\usepackage{paralist}
\usepackage{mathtools}

\usepackage{bbm} 
\usepackage{wrapfig}
\usepackage{makecell}
\usepackage{multirow}
\usepackage{booktabs}

\usepackage{nicefrac}       

\usepackage[flushleft]{threeparttable} 

\usepackage{caption}
\usepackage{multirow}
\usepackage{colortbl}
\definecolor{bgcolor}{rgb}{0.8,1,1}
\definecolor{bgcolor2}{rgb}{0.8,1,0.8}
\definecolor{niceblue}{rgb}{0.0,0.19,0.56}

\usepackage{hyperref}
\hypersetup{colorlinks,linkcolor={blue},citecolor={niceblue},urlcolor={blue}}

\usepackage{pifont}
\definecolor{PineGreen}{RGB}{0,110,51}
\definecolor{BrickRed}{RGB}{143,20,2}

\usepackage{tikz-cd} 

\DeclareMathOperator*{\argmin}{arg\,min}

\newcommand{\R}{\mathbb{R}}

\def\<#1,#2>{\left\langle #1,#2\right\rangle}

\newcolumntype{Y}{>{\centering\arraybackslash}X}

\usepackage{xspace}

\usepackage[colorinlistoftodos,bordercolor=orange,backgroundcolor=orange!20,linecolor=orange,textsize=scriptsize]{todonotes}

\usepackage{hyperref}
\graphicspath{{plots/}}

\usepackage{makecell}

\usepackage{accents}
\newlength{\dhatheight}

\usepackage{pgfplotstable} 
\usetikzlibrary{automata, positioning, arrows, shapes, fit, calc, intersections}
\usepgfplotslibrary{statistics}

\def\la{\langle}
\def\ra{\rangle}

\newtheorem{theorem}{Theorem}

\begin{document}
\title{\bf\large\MakeUppercase{ On the Equivalence of Optimal Transport Problem and Action Matching with Optimal Vector Fields}}
\author{Nikita Kornilov$^{1,2}$, Alexander Korotin$^{1}$\footnote{1. Applied AI Institute, Moscow, Russia 2. Moscow Center for Advanced Studies, Russia}}
\date{}

\maketitle
 \paragraph{Abstract.} Flow Matching (FM) method in generative modeling maps arbitrary probability
distributions by constructing an interpolation between them and then learning
the vector field that defines ODE for this interpolation. Recently, it was
shown that FM can be modified to map distributions optimally in terms of the
quadratic cost function for any initial interpolation. To achieve this, only
specific optimal vector fields, which are typical for solutions of Optimal
Transport (OT) problems, need to be considered during FM loss minimization. In this note, we show
that considering only optimal vector fields can lead to OT in another approach:
Action Matching (AM). Unlike FM, which learns a vector field for a manually
chosen interpolation between given distributions, AM learns the vector field
that defines ODE for an entire given sequence of distributions.

\section{Background on related works} 
In this section, we briefly recall the most related methods: Flow Matching method \cite{liu2023flow} that maps arbitrary distributions to each other, Action Matching  method \cite{neklyudov2023action}  that learns the whole given dynamics between distributions and Optimal Flow Matching method \cite{kornilov2024optimal} that modifies FM to map the distributions optimally in terms of the quadratic transport cost.     

\paragraph{Flow Matching (FM) \cite{liu2023flow}.} FM maps distribution $p_0$ to $p_1$ by integrating a specific vector field of an ordinary differential equation (ODE) over a time interval \([0,1]\). It builds this specific vector field $u^\pi_{FM}: [0,1] \times \mathbb{R}^D \to \mathbb{R}^D$ by minimizing the following loss \(\mathcal{L}^\pi_{FM}(u)\) over vector fields \(u\), namely, 
$$\mathcal{L}^\pi_{FM}(u) := \int_{0}^1 \int_{\mathbb{R}^D \times \mathbb{R}^D} \| u_t(x_t) - (x_1 - x_0) \|^2 \pi(x_0, x_1) \, dx_0 \, dx_1 \, dt,$$
where \(\pi\) is a given transport plan between \(p_0\) and \(p_1\) (e.g., \(\pi = p_0 \times p_1\)). If the obtained field \(u^\pi_{FM} = \argmin_u \mathcal{L}^\pi_{FM}(u)\) is integrated starting from \(x_0 \sim p_0\), the distribution of the endpoint \(x_1(x_0)\) at time \(t=1\) will exactly match \(p_1\), meaning that this mapping is indeed a transport one. Note that the properties of the final mapping depend on the initial plan \(\pi\).

\paragraph{Action Matching (AM) \cite{neklyudov2023action}.} AM works with a sequence of distributions \( \{p_t\}_{t \in [0,1]} \), starting at \( p_0 \) and ending at \( p_1 \). The method builds an ODE vector \( u^{\{p_t\}}_{AM}: [0,1] \times \mathbb{R}^D \to \mathbb{R}^D \) which generates the whole probability sequence \(  \{p_t\}_{t \in [0,1]}  \).
For that it minimizes the following loss $\mathcal{L}^{\{p_t\}}_{AM}(s)$ over scalar functions \( s_t: [0,1] \times \mathbb{R}^D \to \mathbb{R} \), namely,
 \begin{eqnarray}
      \mathcal{L}^{\{p_t\}}_{AM}(s) &=& \int_{\R^D} s_0(x_0) p_0(x_0) dx_0 - \int_{\R^D} s_1(x_1) p_1(x_1) dx_1\notag \\
      &+& \int_{0}^1 \int_{\R^D}\! \left[\frac{1}{2}\|\nabla s_t(x_t)\|^2 + \frac{\partial s_t}{\partial t} (x_t) \right]\! p_t(x_t) dx_t dt \label{eq:AM loss}.
\end{eqnarray}
For the minimizer \( s^{\{p_t\}}_{AM} := \argmin_s \mathcal{L}^{\{p_t\}}_{AM}(s) \), the corresponding vector field \( u^{\{p_t\}}_{AM}\) is defined as $u^{\{p_t\}}_{AM}(x_t, t) := \nabla_{x_t} s^{\{p_t\}}_{AM}(x_t, t).$ Integrating this field starting from \( x_0 \sim p_0 \) generates a process which marginal distributions at each time \( t \) coincide with \( p_t \). In particular, at \( t=1 \), the integration yields a transport map from \( p_0 \) to \( p_1 \) which properties depend on  \( \{ p_t \} \).

\paragraph{Optimal Transport (OT) \cite{lol}}. Among all possible transport maps from \( p_0 \) to \( p_1 \), there exists one that optimally preserves proximity between input and output points in the mean-square sense. This map is known as the solution to the Optimal Transport (OT) problem with the quadratic cost function. One way to compute it is by minimizing the dual OT loss $\mathcal{L}_{OT}(\Psi)$ over \textit{convex} functions \( \Psi: \mathbb{R}^D \to \mathbb{R} \):  
\begin{eqnarray}
    \mathcal{L}_{OT}(\Psi) := \int_{\mathbb{R}^D} \Psi(x_0) p_0(x_0) \, dx_0 + \int_{\mathbb{R}^D} \overline{\Psi}(x_1) p_1(x_1) \, dx_1, \label{eq: OT loss} 
\end{eqnarray}
where \( \overline{\Psi}(x_1) := \sup_{y \in \mathbb{R}^D} \left[ \langle x_1, y \rangle - \Psi(y) \right] \) is the convex conjugate of \( \Psi \). The optimal function \( \Psi^* \) is called the \textit{Brenier potential} \cite{villani2021topics}, and the OT map itself is given by \( x_0 \mapsto \nabla \Psi^*(x_0) \).

\paragraph{Optimal Flow Matching (OFM) \cite{kornilov2024optimal}.} The authors of \cite{kornilov2024optimal} propose OFM method, an adaptation of FM to solve the OT problem. They restrict the domain of FM loss minimization only to \textit{optimal vector fields} which are typical for the solutions of the dynamic OT or the Benamou-Brenier problem. These fields are defined by convex functions \( \Psi \):  initial point \( z_0 \) at \( t=0 \) moves to \( \nabla \Psi(z_0) \) at \( t=1 \), and intermediate point \( x_t, t \in (0,1) \) lies on the straight line between \( z_0 \) and \( \nabla \Psi(z_0) \): \[
  x_t = (1-t) z_0 + t \nabla \Psi(z_0), \quad u^{\Psi}_t(x_t) = \nabla \Psi(z_0) - z_0.
  \]
 With notation $\mathcal{L}_{OFM}^\pi(\Psi) := \mathcal{L}_{FM}^\pi(u^\Psi)$, 
for any  plan $\pi$, the following relation holds $$\mathcal{L}^\pi_{OFM}(\Psi)=2\mathcal{L}_{OT}(\Psi)+\text{Const}(\pi).$$
It means that the losses $\mathcal{L}^\pi_{OFM}$ and $\mathcal{L}_{OT}$ share the same minimizers w.r.t. $\Psi$. As a consequence, unlike FM, the solution $\Psi^* = \argmin_\Psi \mathcal{L}^\pi_{OFM}(\Psi)$ in the OFM method does not depend on the initial plan $\pi$, and $\nabla \Psi^{*}$ yields the OT map. \newline
\vspace{-4mm}

\section{Equivalence of Optimal Transport and Action Matching }

In this section, we demonstrate that Action Matching loss \eqref{eq:AM loss}, when considering only the optimal vector fields \( u^{\Psi} \) (more precisely their potentials \( s^{\Psi} \)), is  equivalent to the OT dual loss \eqref{eq: OT loss} for any sequence \( \{p_t\} \).
\begin{theorem}
      Consider two probability distributions  $p_0, p_1\in\mathcal{P}_{2,ac}(\mathbb{R}^{D})$ and any sequence of distributions  $\{p_t\}_{t \in (0,1)}$ connecting them. Then, OT dual loss  $\mathcal{L}_{OT}(\Psi)$ and AM loss $\mathcal{L}^{\{p_t\}}_{AM}(s^{\Psi})$ match each other up to a constant.
\end{theorem}
\begin{proof}
First, we find an explicit formula for $s^\Psi$ which gradient equals $u^\Psi$, i.e., $u^\Psi_t \equiv \nabla s_t^\Psi.$ We note that, for point $x_t \in \R^D$,  initial  point $z_0 \in \R^D$ satisfies \vspace{-1mm}
 \begin{eqnarray}
 x_t  &=&  t \nabla\Psi(z_0)  + (1-t) z_0 \Rightarrow x_t = \nabla ( t \Psi(\cdot) + (1-t)\|\cdot\|^2/2 )(z_0) := \nabla \varphi_t(z_0).\notag
 \end{eqnarray}
Then, $z_0 = \nabla \overline{\varphi_t}(x_t)$ and vector field $u^{\Psi}_t(x_t)$ can be expressed via $s^\Psi$: 
 \begin{eqnarray}
     u^{\Psi}_t(x_t) &=& \nabla\Psi(z_0) - z_0 = \frac{x_t - z_0}{t} = \nabla (\frac{\|\cdot\|^2}{2t}  - \frac{\overline{\varphi_t}}{t})(x_t), \notag \\ 
     s^\Psi_t(x_t) &:=& \frac{\|x_t\|^2}{2t}  - \frac{\overline{\varphi_t}(x_t)}{t} .\label{eq:s_t explicit}
 \end{eqnarray}
In particular, the corner cases $t = 0$ and $t = 1$ are $s^\Psi_1(x_1) =  \nicefrac{\|x_1\|^2}{2} - \overline{\Psi}(x_1), \quad s^\Psi_0(x_0) =  \Psi(x_0) - \nicefrac{\|x_0\|^2}{2}.  $
Next, we put explicit values of the function $s^\Psi$ from \eqref{eq:s_t explicit} into AM functional \eqref{eq:AM loss}. For the time $t \in (0,1)$, the following equality holds $\|\nabla s^\Psi_t(x_t)\|^2/2 = \nicefrac{\|x_t  - z_0\|^2}{2t^2}.$ If we take derivative from $s_t$ w.r.t. $t$, the term  $\|x_t\|^2/2t$ from \eqref{eq:s_t explicit} can be omitted:
\begin{eqnarray}
    \frac{\overline{\varphi_t}(x_t)}{t} &=& \frac{1}{t} \max\limits_{z \in \R^D}\left\{ \la x_t, z \ra  - t \Psi(z) - \frac{(1-t)}{2}\|z\|^2\right\} \notag \\
    &=& \max\limits_{z \in \R^D}\left\{ \frac{
 \la x_t, z \ra}{t}  -  \Psi(z) - \frac{(1-t)}{2t}\|z\|^2\right\}.
\end{eqnarray}
Moreover, $\max$ is attained at the point $z_0.$ According to the Envelope Theorem, in order to take derivative from $\max$ w.r.t. $t$, one needs to differentiate maximized function and then put the point $z_0$, at which maximum is attained, as an argument: \vspace{-0.5mm}
 \begin{eqnarray}
  \frac{\partial (\overline{\varphi_t}/t)}{\partial t}   &=& - \frac{\la x_t, z_0 \ra}{t^2}+ \frac{\|z_0\|^2}{2t^2}, \notag \\     \frac{\partial s^\Psi_t}{\partial t} &= &-\frac{\|x_t\|^2}{2t^2}  + \frac{\la x_t, z_0 \ra}{t^2} -  \frac{\|z_0\|^2}{2t^2} = -\frac{1}{2}\frac{\|x_t  - z_0\|^2}{t^2}.\notag
 \end{eqnarray} 
 Hence, in case of the optimal vector fields, the following holds for any $t \in [0,1]$ and $x_t \in \R^D$: \vspace{-1mm}
$$ \left[\frac{1}{2}\|\nabla s_t^\Psi(x_t)\|^2 + \frac{\partial s^\Psi_t}{\partial t} (x_t) \right] =  \frac{1}{2}\frac{\|x_t  - z_0\|^2}{t^2} -\frac{1}{2}\frac{\|x_t  - z_0\|^2}{t^2} \equiv 0. $$
As a consequence, AM loss \eqref{eq:AM loss} with $s^{\Psi}$ is \textit{completely independent from} $p_t$ for $t\in (0,1)$ and equivalent to dual OT loss \eqref{eq: OT loss} up to the second moments $p_0,p_1$ (constants during minimization):
 $$\mathcal{L}^{\{p_t\}}_{AM}(s^\Psi) = \int_{\R^D} \left(\Psi(x_0) - \frac{\|x_0\|^2}{2}\right) p_0(x_0) dx_0 + \int_{\R^D} \left(\overline{\Psi}(x_1) - \frac{\|x_1\|^2}{2}\right) p_1(x_1) dx_1. \vspace{-4mm}$$
 
\end{proof}

\noindent\textit{Similar to Flow Matching, the optimal vector fields in Action Matching always lead to the Optimal Transport solution \(\Psi^*\) regardless of the initial sequence \(\{p_t\}\). This highlights the potential of such fields in developing robust methods for solving OT problems in practice.}


\bibliographystyle{plain}\vspace{-3mm}
\bibliography{Bibliography}

\end{document}